\documentclass{article}

\input{definitions}

\usepackage{microtype}
\usepackage{graphicx}
\usepackage{subfigure}
\usepackage{booktabs} 

\usepackage{hyperref}
\usepackage{url}
\usepackage{algorithm}
\usepackage[noend]{algorithmic} 

\usepackage{amsmath,amsfonts,amsthm}
\usepackage{thmtools}
\usepackage{mathtools}
\usepackage{multirow}
\usepackage{framed}
\usepackage{enumitem}
\usepackage{bbm}
\newtheorem{theorem}{Theorem}
\newtheorem{lemma}{Lemma}
\newtheorem{remark}{Remark}
\newtheorem{fact}{Fact}
\newtheorem{proposition}{Proposition}

\usepackage{hyperref}

\usepackage[accepted]{icml2021}
\icmltitlerunning{Speeding up Deep Learning Training by Sharing Weights and Then Unsharing}

\begin{document}

\twocolumn[
\icmltitle{Speeding up Deep Model Training by Sharing Weights and Then Unsharing}

\icmlsetsymbol{equal}{*}

\begin{icmlauthorlist}
\icmlauthor{Shuo Yang}{equal,aus}
\icmlauthor{Le Hou}{equal,goo}
\icmlauthor{Xiaodan Song}{goo}
\icmlauthor{Qiang Liu}{aus}
\icmlauthor{Denny Zhou}{goo}
\end{icmlauthorlist}

\icmlaffiliation{aus}{Department of Computer Science, The University of Texas at Austin}
\icmlaffiliation{goo}{Google}

\icmlcorrespondingauthor{Denny Zhou}{dennyzhou@google.com}
\icmlkeywords{Neutral network training, Optimization for deep learning, BERT}

\vskip 0.3in
]

\printAffiliationsAndNotice{\icmlEqualContribution~Work done at Google. Submitted to ICLR 2021.} 

\begin{abstract}
We propose a simple and efficient approach for training the BERT model. Our approach exploits the special structure of BERT that contains a stack of repeated modules (i.e., transformer encoders). Our proposed approach first trains BERT with the weights shared across all the repeated modules till some point. This is for learning the commonly shared component of weights across all repeated layers. We then stop weight sharing and continue training until convergence. We present theoretic insights for training by sharing weights then unsharing with analysis for simplified models. Empirical experiments on the BERT model show that our method yields better performance of trained models, and significantly reduces the number of training iterations.
\end{abstract}

\section{Introduction}
\label{sec:introduction}
It has been widely observed that increasing model size often leads to significantly better performance on various real tasks, especially natural language processing applications \citep{amodei2016deep, wu2016google, vaswani2017attention, devlin2018bert, raffel2020exploring, brown2020language}. However, as models getting larger, the optimization problem becomes more and more challenging and time consuming. To alleviate this problem, there has been a growing interest in developing systems and algorithms for more efficient training of deep learning models, such as distributed large-batch training \citep{goyal2017accurate, shazeer2018mesh,lepikhin2020gshard, you2020large}, various normalization methods \citep{de2020batch, salimans2016weight, ba2016layer}, gradient clipping and normalization \citep{kim2016accurate, chen2018gradnorm}, and progressive/transfer training \citep{chen2015net2net, chang2018multi, gong2019efficient}.

In this paper, we seek a better training approach by exploiting common network architectural patterns. In particular, we are interested in speeding up the training of deep networks which are constructed by repeatedly stacking the same layer, with a special focus on the BERT model. We propose a simple method for efficiently training such kinds of networks. In our approach, we first force the weights to be shared across all the repeated layers and train the network, and then stop the weight-sharing and continue training until convergence. Empirical studies show that our method yields trained models with better performance, or reduces the training time of commonly used models.

Our method is motivated by the successes of weight-sharing models, in particular, ALBERT \citep{lan2020albert}. It is a variant of BERT in which weights across all repeated transformer layers are shared. As long as its architecture is sufficiently large, ALBERT is comparable with the original BERT on various downstream natural language processing benchmarks. This indicates that the optimal weights of an ALBERT model and the optimal weights of a BERT model can be very close. We further assume that there is a \textit{commonly shared component} across weights of repeated layers.

Correspondingly, our training method consists of two phases: In the first phase, a neural network is trained with weights shared across its repeated layers, to learn the commonly shared component across weights of different layers; In the second phase, the weight-sharing constraint is released, and the network is trained till convergence, to learn a different weight for each layer based on the shared component.

We theoretically show that training in the direction of the shared component in the early steps effectively constrains the model complexity and provides better initialization for later training. Eventually, the trained weights are closer to the optimal weights, which leads to better performance. See Figure \ref{fig:illustration} for an illustration.

\begin{figure}[ht]
    \includegraphics[clip, trim={220 200 340 120}, width=\columnwidth]{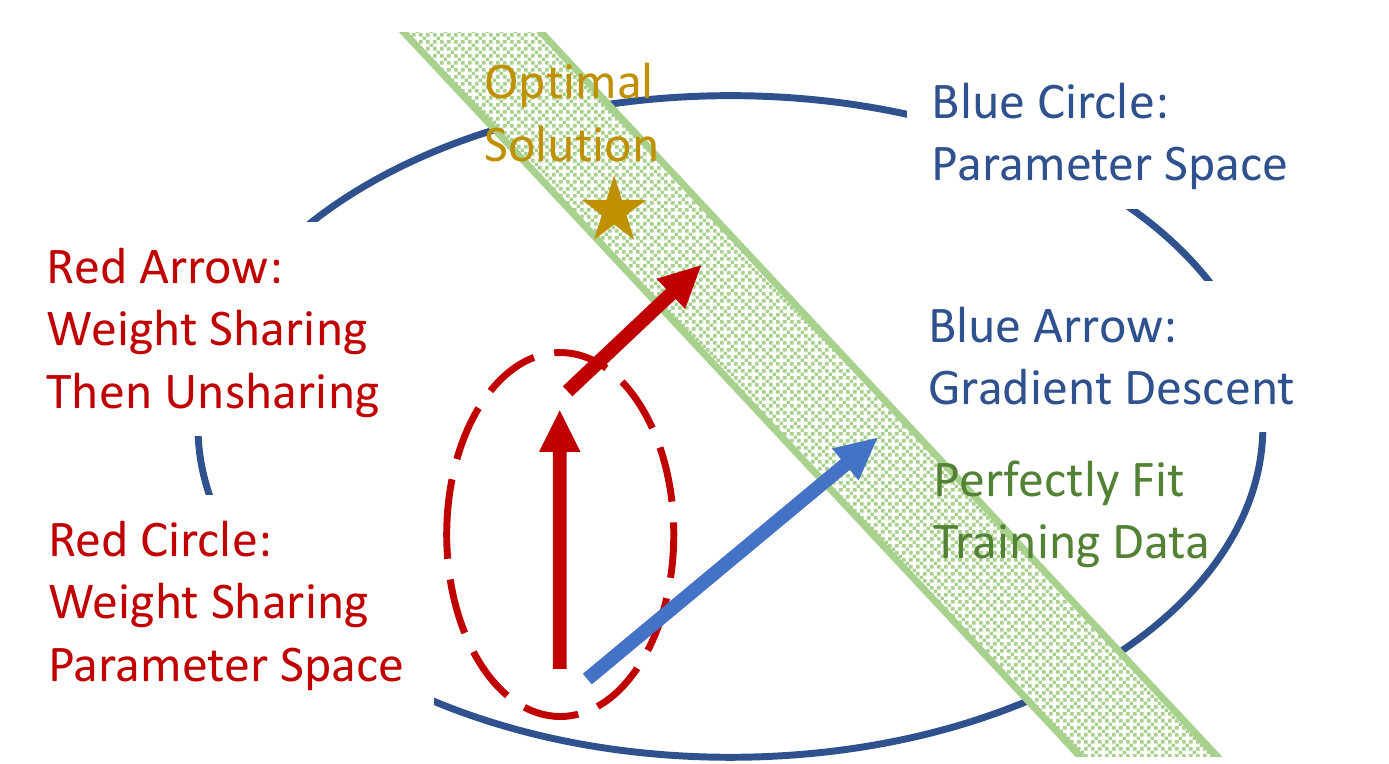}
    \vspace*{-6mm}
    \caption{Intuition for sharing weights then unsharing. The figure illustrates how the weights change when trained by sharing weights then unsharing (the red arrows) and by gradient descent (the blue arrow). Weight sharing first trains the weights in a more constrained set. It brings the weights closer to the optimal solution, which provides a better initialization for subsequent training. Comparing with gradient descent, which directly optimizes the weights in the original parameter space, our proposed training method converges to a solution closer to the optimal, and thus leads to better generalization performance.
    }
    \label{fig:illustration}
\end{figure}

Existing approaches can be viewed as extreme cases as our proposed method: ALBERT shares weights throughout the training process, and BERT does not share weights at all. Our experiments show that oftentimes under various settings, the best result is given by sharing weights for the first 10\% of the total number of iterations.

The rest of this paper is organized as follows. We present our algorithmic motivation and the sharing weights then unsharing algorithm in Section \ref{sec:algorithm}. In Section \ref{sec:theory}, we present some theoretical insights of the proposed algorithm. In Section \ref{sec:related}, we discuss related work.  In Section \ref{sec:experiments},  we show detailed experimental setups and results. We also provide various ablation studies on different choices in implementing our algorithm. Finally, we conclude this paper with discussions in Section \ref{sec:conclusion}. 
\section{Sharing Weights and Then Unsharing}
\label{sec:algorithm}
In this section, we first motivate the training with ``sharing weights and then unsharing" as a way of reparameterization. We show how a particular reparameterization of ``stem-direction" and ``branch-direction" naturally leads to our algorithm. We then formally present our algorithm.

\paragraph{Algorithmic Motivation}
For a machine learning model $\Mcal$ composed by $L$ modules with the same structure $\cbr{\Mcal_{i}}, i = 1,\cdots, L$ (e.g., the transformer modules \citep{vaswani2017attention} in BERT), let $w$ be the trainable weights of $\Mcal$ and $w_{i}$ be the weights correspond to $\Mcal_{i}$. We can rewrite the $w_{i}$s as
\begin{align}\label{eq:reparameter}
  w_{i} = \frac{1}{\sqrt{L}}w_{0} + \widetilde w_{i}, \quad i = 1, \cdots, L
.\end{align}
The $w_{0}$ can be viewed as the \textbf{\textit{stem-direction}} shared by all $w_{i}$, while $\widetilde w_{i}$s are the \textbf{\textit{branch-directions}}, capturing the difference among $w_{i}$s. The $\frac{1}{\sqrt{L}}$ is a scaling factor whose meaning will be clear soon. To optimize $w$ for $\Mcal$ by $T$ steps of gradient descent, let $\eta$ be the step size, one could do 

\begin{enumerate}
  \item \textbf{Sharing weights in early stage: }For some $\alpha \in (0,1)$, in the first $\tau = \alpha\cdot T$ steps, compute gradient $g_{0}$ for $w_{0}$ and update all $w_{0}$ by $\eta g_{0}$.
  \item \textbf{Unsharing weights in later stage: } For $t \ge \alpha \cdot T$, compute gradient $\widetilde g_{i}$ for $\widetilde w_{i}$ and update all $\widetilde w_{i}$ with $\eta \widetilde g_{i}$.
\end{enumerate}
On a high-level, training only on the \textit{stem-direction} ($w_{0}$) in the early steps effectively constrains the model complexity and provides better initialization for later training.

It is very easy to implement the training with aforementioned reparameterization. The gradients $g_0, \widetilde g_i$ can be easily adapted from the gradient $g_i$ of weights $w_i$, where we have
\begin{align*}
    & g_{0} = \sum_{i=1}^{L}\frac{\partial w_{i}}{\partial w_{0}}g_{i} = \frac{1}{\sqrt{L}}\sum_{i=1}^{L} g_{i},\quad \widetilde g_{i} = \frac{\partial w_{i}}{\partial \widetilde w_{i}}g_{i} = g_{i}.
\end{align*}

Thus the effective update to $w_i$ in every step are
\begin{align*}
    & \textbf{Sharing weights : } \frac{1}{\sqrt{L}}\cdot \frac{\eta}{\sqrt{L}}\sum_{i=1}^{L}g_{i} = \eta \cdot \frac{1}{L}\sum_{i=1}^{L}g_{i}  \\
    & \textbf{Unsharing weights : } \eta g_{i}
\end{align*}

In the special case of all $g_i$s are equal, the ``sharing weights and then unsharing" is equivalent to the standard gradient descent. This shows that $\frac{1}{\sqrt{L}}$ gives the right scaling.
 
\paragraph{Practical Algorithm}
We now formally present our algorithm, which first trains the deep network with all the weights shared. Then, after a certain number of training steps, it unties the weights and further trains the network until convergence. See Algorithm \ref{alg:weight_sharing_fix}.

\begin{algorithm}[ht]
\centering 
\begin{algorithmic}[1]
\STATE \textbf{Input:} total number of training steps $T$, untying point $\tau$, learning rates $\{\eta^{(t)}, t = 1, \dots, T\}$
\STATE  Randomly and equally initialize weights $w^{(0)}_1, \dots, w^{(0)}_L$  
\FOR{$t=1$ \textbf{to} $T$}
    \STATE Compute gradient $g^{(t)}_i$ of $w^{(t-1)}_i, \ i=1, \dots, L$ 
    \IF{$t < \tau$}
        \STATE $w^{(t)}_i = w^{(t-1)}_i - \frac{\eta^{(t)}}{L}\sum_{k=1}^L g^{(t)}_{k},~i = 1, \dots, L$
    \ELSE
        \STATE $w^{(t)}_i = w^{(t-1)}_i - \eta^{(t)} \cdot g_i^{(t)}, \ i = 1, \dots, L$
    \ENDIF
\ENDFOR 
\end{algorithmic}
\caption{\textsc{Sharing Weights and then Untying} }
\label{alg:weight_sharing_fix}
\end{algorithm}

Note that, from lines 1 to 6, we initialize all the weights equally,  and then update them using the mean of their gradients. It is easy to see that such an update is equivalent to sharing weights; in lines 6 to 8, the $w_i^{(t)}$s are updated according to their own gradient, which corresponds to the unshared weight. For the sake of simplicity,  we only show how to update the weights using the plain (stochastic) gradient descent rule. One can replace this plain update rule with any of their favorite optimization methods, for example, the Adam optimization algorithm \citep{kingma2014adam}. 

While the repeated layers being the most natural units for weight sharing, that is not the only choice. We may view several layers together as the weight sharing unit, and share the weights across those units. The layers within the same unit can have different weights. For example, for a 24-layer transformer model, we may combine every four layers as a weight-sharing unit.  Thus,  there will be six such units for weight sharing. Such flexibility of choosing weight-sharing units allows for a balance between ``full weight sharing" and ``no weight sharing" at all.
\section{Theoretic Understanding}
\label{sec:theory}

In this section, we first provide theoretic insights of sharing weights then unsharing via an illustrative example of over-parameterized linear regression. Both the numerical experiment and theoretical analysis show that training by weight sharing allows the model to learn the commonly shared component of weights, which provides good initialization for subsequent training and leads to good generalization performance. We then extend the analysis to a deep linear model and show similar benefits.

\subsection{An Illustrative Example}
We first use linear regression, the simplest machine learning model, to illustrate our core idea - training with weight sharing helps the model learn the commonly shared component of the optimal solution. Consider a linear model with $y = w^{*\top}x$, where $w^{*}\in \RR^{d}$ is the ground truth model parameter; $x\in \RR^{L}$ is the input and $y$ is the response. We can rewrite $w^{*}$ element-wise as $w^{*}_{i} = w^{*}_{0} + \widetilde w^{*}_{i}$, i.e., elements in $w$ are roughly analogous to the repeated modules.

As a concrete example, we set $L = 200$, and draw $w^{*} \sim \Ncal^{200}(1, 1)$ and generate 120 training samples by drawing $x_i \sim \Ncal^{200}(0, 1)$ and compute $y_i = w^{*\top}x_i$ correspondingly. Note that the dimension is larger than the size of training samples, and the regression problem yields an infinite number of solutions that can perfectly fit the training data.

We initialize $\widehat w = 0$ and train the model with \textit{gradient descent, sharing weights then unsharing} for 500 steps, with weights untied after 100 steps. Figure \ref{fig:regression} shows that training by sharing weights then unsharing, the model first learns the mean of $w^*$ and has significantly better performance when fully trained.

\begin{figure}[ht]
    \includegraphics[width=\columnwidth]{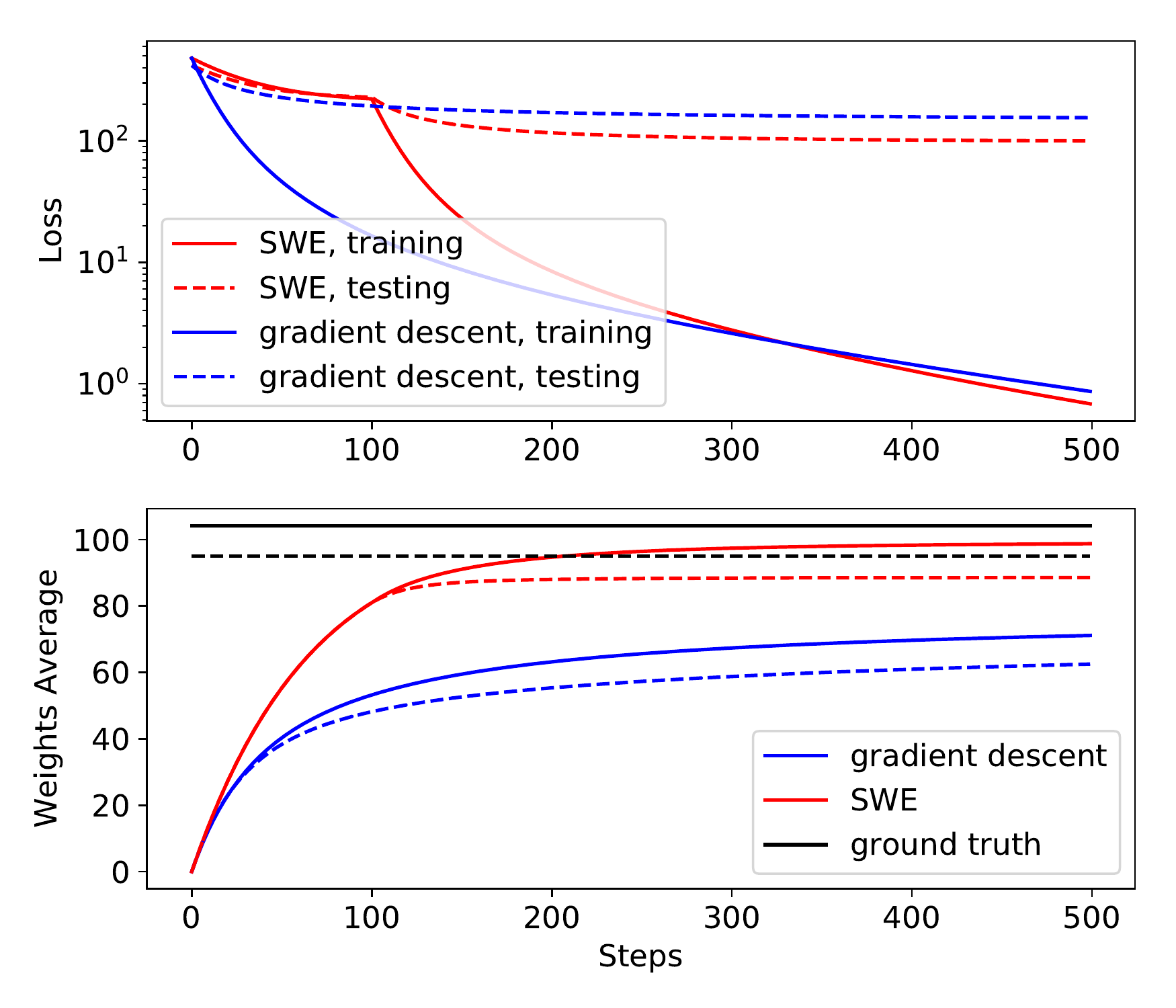}
    \vspace*{-6mm}
    \caption{Fitting an under-determined linear regression (less training samples than dimension). ``SWE" refers to training by ``sharing weights then unsharing". The first figure shows that both methods perfectly fit the training set, while the sharing weights then unsharing has significantly better generalization performance. The second shows the change of parameters. The solid and dash curves are the average of the first and last 100 elements in $\widehat w$. By
    training with weight-sharing, $\widehat w$ first learns the mean of $w^{*}$ which leads to a much better initialization when it starts to learn the whole $w^{*}$.}
    \label{fig:regression}
\end{figure}

Next, we formalize our observation in Figure \ref{fig:regression} that training with sharing weights then unsharing learns the commonly shared component, even when the model is over-parameterized (i.e. having more parameters than training samples).

Denote by $\Wcal_{0}$  the space generated by the shared weights $[w_{0}, \cdots, w_{0}]\in \RR^L$. In the special case of the aforementioned linear regression, $\Wcal_0$ is a 1-dimensional space. Let $\overline{w}^* = \Pcal_{\Wcal_{0}}(w^*)$ be the commonly shared weights for all coordinates, where $\Pcal_{\Wcal_0}(\cdot)$ denotes the projection to $\Wcal_0$. Here the elements of $\overline w^*$ are simply the average of $w^*$. Our next proposition shows that when trained with weight sharing, $\widehat w$ converges to $\overline{w}^*$.

\begin{proposition}\label{prop:regression}
  Consider the $L$-dimensional linear regression problem and $n$ training samples $\cbr{x_i, y_i}$ with $x_i$ generated from $\Ncal^L(0, 1)$ and $y_i = x_i^\top w^*$, we have
  \begin{align*}
      \norm{\widehat w - \overline w^*}_2  = O\rbr{\sqrt{L/n}},
  \end{align*}
  where $\widehat w$ is the solution obtained by training with weight sharing. Especially, this result holds for the over-parameterized regime with $L > n$.
\end{proposition}

As an example, for $n = L/ 2$, we have $\norm{\widehat w - \overline w^*}_2$ being a dimension-independent constant, whereas $\norm{\overline w^*}_2$ scales as $\Theta(\sqrt{L})$. It shows that training with weight sharing effectively constrains the model complexity, and learns the common component $\overline w^*$ of the optimal solution $w^*$ even in the over-parameterized regime.

Further, let $\Wcal_0^{\perp}$ be the space orthogonal to $\Wcal_0$ and let $\widetilde w^{*} = \Pcal_{\Wcal_0^\perp}(w^*)$, we can decompose the error in the parameter space as
\begin{align*}
  \norm{\widehat w - w^{*}}^{2} = \norm{\Pcal_{\Wcal_{0}}(\widehat w) - \overline w^{*}}^{2} + \norm{\Pcal_{\Wcal_{0}^{\perp}}(\widehat w) - \widetilde w^{*}}^{2}
.\end{align*}
Therefore, training with weight sharing corresponds to minimizing the first term $ \norm{\Pcal_{\Wcal_{0}}(\widehat w) - \overline w^{*}}^{2}$, which brings $\widehat w$ closer to the optimal solution $w^*$, even when the model is over-parameterized. For the subsequent training, the weights are untied and the model is then trained in the original parameter space $\Wcal$. The solution we obtained in $\Wcal_0$ (by weight sharing) can be viewed as a good initialization in the parameter space $\Wcal$.

\subsection{Further Implication for Deep Linear Model}

Here we further show the theoretic insights of sharing weights then unsharing with deep linear model as an example. In particular, we show that training by weight sharing can 1) speed up the convergence, and 2) provides good initialization to subsequent training.

A deep linear network is a series of matrix multiplications
\begin{align*}
    f(\xb; W_1, ..., W_L) = W_LW_{L-1}...W_1\xb,
\end{align*}
with $W_l \in \RR^{d\times d},~~~\ell=1,\ldots,L.$ At the first glance, deep learning models may look trivial since a deep linear model is equivalent to a single matrix. However, when trained with backpropagation, its behavior is analogous to generic deep models. 

The task is to train the deep linear network to learn a target matrix $\Phi\in \RR^{d\times d}$. To focus on the training dynamics, we adopt the simplified objective function
\begin{align*}
    \Rcal(W_1,...W_L) = \frac{1}{2}\norm{W_LW_{L-1}...W_2W_1-\Phi}_F^2.
\end{align*}

Let $\nabla_l\Rcal$ be the gradient of $\Rcal$ with respect to $W_l$. We have
\begin{align*}
    \nabla_l\Rcal = \frac{\partial\Rcal}{\partial W_l} = W_{L:l+1}^T(W_{L:1}-\Phi)W_{l-1:1}^T,
\end{align*}
where $W_{l_2:l_1} = W_{l_2}W_{l_2-1}...W_{l_1+1}W_{l_1}$. The standard gradient update is given by
\begin{align*}
    W_l(t+1) = W_l(t) - \eta\nabla_l\Rcal(t), \quad l = 1,...,L.
\end{align*}
To train with weights shared, all the layers need to have the same initialization. And the update is
\begin{align}\label{eqn:weight_sharing_update}
    W_l(t+1) = W_l(t) - \frac{\eta}{L}\sum_{i=1}^L\nabla_i\Rcal(t), \quad l = 1,...,L.
\end{align}

\textbf{Sharing Weights Brings Faster Convergence}

Since the initialization and updates are the same for all layers, the parameters $W_1(t),...,W_L(t)$ are equal for all $t$. For simplicity, we denote the weight at time $t$ to be $W_0(t)$. Notice that the gradients are averaged, the norm of the update to each layer doesn't scale with $L$.

We first consider the case where the target matrix $\Phi$ is positive definite. It is immediate that $\Phi^{1/L}$ is a solution to the deep linear network. We study the convergence result with continuous-time gradient descent (with extension to discrete-time gradient descent deferred to appendix), which demonstrates the benefit of training with weight sharing when learning a positive definite matrix $\Phi$. We draw a comparison with training with zero-asymmetric (ZAS) initialization \citep{wu2019global}. To the best of our knowledge, ZAS gives the state-of-the-art convergence rate. It is actually the only work showing the global convergence of deep linear networks trained by gradient descent.

With continuous-time gradient descent (i.e. $\eta \rightarrow 0$), the training dynamics of continuous-time gradient descent can be described as
\begin{align*}
    \frac{d W_l(t)}{dt} = \dot W_l(t) = -\nabla_l\Rcal(t), \quad l = 1,...,L, \quad t\ge 0.
\end{align*}
\begin{fact}[Continuous-time gradient descent without weight sharing \citep{wu2019global}]\label{thm:no_weight_sharing}
For the deep linear network $f(\xb; W_1, ..., W_L) = W_LW_{L-1}...W_1\xb$, the continuous time gradient descent with the zero-asymmetric initialization satisfies $\Rcal(t) \le \exp\rbr{-2t}\Rcal(0).$
\end{fact}
Fact \ref{thm:no_weight_sharing} shows that with the zero-asymmetric initialization, the continuous gradient descent linearly converges to the global optimal solution for general target matrix $\Phi$.

We have the following convergence result for training with weight sharing.

\begin{theorem}[Continuous-time gradient descent with weight sharing]\label{thm:weight_sharing} For the deep linear network $f(\xb; W_1, ..., W_L) = W_LW_{L-1}...W_1\xb$, initialize all $W_l(0)$ with identity matrix $I$ and update according to Equation \ref{eqn:weight_sharing_update}. With a positive definite target matrix $\Phi$, the continuous-time gradient descent satisfies $\Rcal(t) \le \exp\rbr{-2L\min(1, \lambda^2_\text{min}(\Phi))t}\Rcal(0)$.
\end{theorem}

Training with ZAS, the loss decays as $\Rcal(t) \le \exp(-2t)\Rcal(0)$, whereas for training with weight sharing, the loss is $\Rcal(t) \le \exp(-2L\min(1, \lambda_\text{min}(\Phi))t)\Rcal(0)$. The extra $L$ in the exponent demonstrates the acceleration of training with weight sharing. See
Appendix \ref{subsec:discrete_time} for the extension to discrete-time gradient descent.

\begin{remark}
  The difference between convergence rates in Fact \ref{thm:no_weight_sharing} and Theorem \ref{thm:weight_sharing} is not an artifact of analysis. For example, when the target matrix is simply $\Phi = \alpha I, \alpha > 1$. It can be explicitly shown that with the initialization in Fact \ref{thm:no_weight_sharing}, we have $\dot\Rcal(0) = -2\Rcal(0)$ while training with weight sharing (Theorem \ref{thm:weight_sharing}), we have $\dot\Rcal(0) = -2L\Rcal(0)$. This implies that the convergence results in Fact \ref{thm:no_weight_sharing} and Theorem \ref{thm:weight_sharing} cannot be improved in general.
\end{remark}

\textbf{Sharing Weights Provides Good Initialization}

In this subsection, we show that weight sharing can provably improve the initialization for training deep linear models, which can bring significant improvement on existing convergence results. Now we consider the case for which the target matrix $\Phi$ may not be a positive definite matrix. In the rest of the analysis, we denote $\Rcal(\widehat W)$ to be the loss induced by parameter $\widehat W$. The local convergence result has been established as
\begin{fact}[Theorem 1 of \citep{arora2018convergence}]\label{fact:local_convergence}
  For any initialization $\widehat W$ with $\norm{\widehat W_{i+1}^{\top}\widehat W_{i+1} - \widehat W_{i}\widehat W_{i}^{\top}}_{F} \le \delta, \forall i$ and $\sqrt{2 \Rcal(\widehat W)} \le \sigma_{min}\rbr{\Phi} - c$ for some $c > 0$. With proper choice of step size $\eta$, training $\widehat W$ with gradient descent converges to $\Phi$.
\end{fact}
This result shows that starting with a good initialization, i.e. layer-wise similar and small initial loss, the deep linear model converges to the optimal solution. Next we show that training with weight-sharing navigates itself to a good initialization and easily improves the above result. We take reparameterization $W_{i} = \frac{1}{\sqrt{L}}\cdot W_0 + \widetilde W_{i},$
where $W_0= \frac{1}{2}\rbr{\Gamma+\Gamma^{\top}}$ is the \textit{stem-direction} forced to be symmetric.
\begin{theorem}[Weight-sharing provably improves initialization]\label{thm:global_convergence} For any target matrix $\Phi$ with relatively small distance to a positive definite matrix $\Phi_{0}$ as $\norm{\Phi - \Phi_{0}}_{F} \le \frac{\sigma_{min}(\Phi_{0})}{3}$. Initialize all $\widehat W_{i}$ to be identity matrix $I^{d}$ and train with weight-sharing, with proper choice of $\beta$, the gradient descent converges to $\Phi$.
\end{theorem}
We show that when trained with weight sharing, the model first learns the commonly shared component of all layers, and converges to the symmetric matrix $\frac{1}{2}\rbr{\Phi + \Phi^\top}$. After untying, the convergence then follows from Fact \ref{fact:local_convergence}, as the weights are already close to the optimal solution. See full proof in the appendix.

It is interesting that weight-sharing easily converts the established local convergence result to global convergence for a large set of target matrix $\Phi$. This demonstrates that training with weight-sharing can find a good initialization, which brings huge benefits to subsequent training. 
\section{Related Work}
\label{sec:related}

\citet{lan2020albert} propose ALBERT with the weights being shared across all its transformer layers.  Large ALBERT models can achieve good performance on several natural language understanding benchmarks. \citet{bai2019trellis} propose trellis networks which are temporal convolution networks with shared weights and obtain good results for language modeling. This line of work is then extended to deep equilibrium models \citep{bai2019deep}  which are equivalent to infinite-depth weight-tied feedforward networks. \citet{dabre2019recurrent} show that the translation quality of a model that recurrently stacks a single layer is comparable to having the same number of separate layers. \citet{zhang2020deeper} also demonstrate the application of weight-sharing in neural architecture search.

There are also a large number of algorithms proposed for the fast training of deep learning models. Net2Net \citep{chen2015net2net} achieves acceleration by training small models first then transferring the knowledge to larger models, which can be viewed as sharing weights within the same layer. Similarly, \citet{chang2018multi} propose to view the residual network as a dynamical system, and start training with a shallow model, and double the depth of the model by splitting each layer into 2 layers and halving the weights in the residual connection. Progressive stacking \citep{gong2019efficient} focuses on the fast training of BERT. The algorithm also starts with training a shallow model, then grows the model by copying the shallow model and stack new layers on top of it. It empirically shows great training acceleration. \citet{dongtowards} demonstrate it is possible to train an adaptively growing network to achieve acceleration. 

\section{Experiments}
\label{sec:experiments}

In this section, we first present the experimental setup and results for training the BERT-large model with and without our proposed Sharing WEights (SWE) method. Then, we show ablation studies of how different untying point values ($\tau$) affect the final performance, and how our method works with different model sizes, etc. In what follows, without explicit clarification, BERT always means the BERT-large model.

\begin{table*}[h]
\centering
\caption{Experiment results using BERT. Models trained with SWE consistently outperforms models trained without SWE.}
\vspace{3mm}
\label{tab:swe_bert}
\begin{tabular}{l|cc|cc|cc|cc}
         & \multicolumn{4}{c|}{English Wikipedia + BookCorpus} & \multicolumn{4}{c}{XLNet data} \\ \cline{2-9}
         & \multicolumn{2}{c|}{0.5 m iterations} & \multicolumn{2}{c|}{1 m iterations} & \multicolumn{2}{c|}{0.5 m iterations} & \multicolumn{2}{c}{1 m iterations} \\
         & Baseline & SWE & Baseline & SWE & Baseline & SWE & Baseline & SWE \\ \hline
Pretrain MLM (acc.\%) & 73.66 & \textbf{73.92} & 74.98 & \textbf{75.09} & 70.06 & \textbf{70.18} & 71.75 & \textbf{71.76} \\ \hline
SQuAD average & 87.27 & \textbf{88.34} & 88.82 & \textbf{89.51} & 88.85 & \textbf{89.20} & 89.93 & \textbf{90.09} \\
GLUE average & 78.17 & \textbf{78.98} & 79.30 & \textbf{80.03} & 79.53 & \textbf{79.97} & 80.20 & \textbf{80.83} \\ \hline
SQuAD v1.1 (F-1\%) & 91.54 & \textbf{92.54} & 92.58 & \textbf{92.81} & 92.41 & \textbf{92.82} & 93.37 & \textbf{93.55} \\
SQuAD v2.0 (F-1\%) & 82.99 & \textbf{84.14} & 85.06 & \textbf{86.20} & 85.28 & \textbf{85.57} & 86.49 & \textbf{86.63} \\
GLUE/AX (corr\%) & 39.0 & \textbf{40.1} & 42.3 & \textbf{43.8} & 43.2 & \textbf{43.7} & 44.2 & \textbf{46.8}  \\
GLUE/MNLI-m (acc.\%) & 85.9 & \textbf{87.2} & 86.9 & \textbf{87.8} & 87.3 & \textbf{88.4} & 88.6 & \textbf{88.8}  \\
GLUE/MNLI-mm (acc.\%) & 85.4 & \textbf{86.7} & 86.1 & \textbf{87.4} & \textbf{87.4} & \textbf{87.4} & 87.9 & \textbf{88.5}  \\
GLUE/QNLI (acc.\%) & 92.3 & \textbf{93.4} & 93.5 & \textbf{94.0} & 91.9 & \textbf{92.8} & 92.3 & \textbf{93.1}  \\
GLUE/QQP (F-1\%) & \textbf{72.1} & 71.7 & 72.2 & \textbf{72.3} & 71.7 & \textbf{72.1} & \textbf{72.3} & 72.1  \\
GLUE/SST-2 (acc.\%) & 94.3 & \textbf{94.8} & 94.8 & \textbf{94.9} & \textbf{95.7} & 95.4 & \textbf{95.9} & 95.7  \\
\end{tabular}
\end{table*}

\subsection{Experimental Setup}
We use the TensorFlow official implementation of BERT \citep{tf_model_garden}. We first show experimental results with English Wikipedia and BookCorpus for pre-training as in the original BERT paper \citep{devlin2018bert}. We then move to the XLNet enlarged pre-training dataset \citep{yang2019xlnet}. We preprocess all datasets with WordPiece tokenization \citep{schuster2012japanese}. We mask 15\% tokens in each sequence. For experiments on English Wikipedia and BookCorpus, we randomly choose tokens to mask. For experiments on the XLNet dataset, we do whole word masking -- in case that a word is broken into multiple tokens, either all tokens are masked or not masked. For all experiments, we set both the batch size and sequence length to 512.

\begin{table*}[h]
\centering
\caption{Experiment results using BERT with alternative embedding sizes, on the English Wikipedia + BookCorpus dataset. For results with the original embedding size of 1024, see Table. \ref{tab:swe_bert}. The proposed SWE approach works for different embedding sizes.}
 \vspace{3mm}
\label{tab:swe_embed_size}
\begin{tabular}{l|cc|cc|cc|cc}
         & \multicolumn{4}{c}{Embedding size = 768} & \multicolumn{4}{|c}{Embedding size = 1536} \\ \cline{2-9}
         & \multicolumn{2}{c|}{0.5 m iterations} & \multicolumn{2}{c|}{1 m iterations} & \multicolumn{2}{c|}{0.5 m iterations} & \multicolumn{2}{c}{1 m iterations} \\
         & Baseline & SWE & Baseline & SWE & Baseline & SWE & Baseline & SWE \\ \hline
Pretrain MLM (acc.\%) & 71.68 & \textbf{72.03} & 73.00 & \textbf{73.10} & 75.61 & \textbf{75.79} & \textbf{77.23} & 77.16 \\ \hline
SQuAD average & 86.65 & \textbf{87.00} & 87.38 & \textbf{87.97} & 88.32 & \textbf{89.68} & 89.38 & \textbf{89.79} \\
GLUE average & 77.30 & \textbf{77.98} & 78.13 & \textbf{79.28} & 79.00 & \textbf{80.08} & 80.37 & \textbf{80.88} \\ \hline
SQuAD v1.1 (F-1\%) & 91.29 & \textbf{91.51} & 91.95 & \textbf{91.99} & 92.29 & \textbf{93.19} & 92.76 & \textbf{93.30} \\
SQuAD v2.0 (F-1\%) & 82.00 & \textbf{82.48} & 82.80 & \textbf{83.95} & 84.35 & \textbf{86.16} & 86.00 & \textbf{86.28} \\
GLUE/AX (corr\%) & 36.4 & \textbf{38.8} & 37.8 & \textbf{42.9} & 41.7 & \textbf{44.2} & 44.6 & \textbf{46.7}  \\
GLUE/MNLI-m (acc.\%) & 85.4 & \textbf{86.0} & 86.1 & \textbf{87.0} & 87.1 & \textbf{87.8} & 88.1 & \textbf{88.5}  \\
GLUE/MNLI-mm (acc.\%) & 85.0 & \textbf{85.9} & 86.0 & \textbf{86.6} & 86.5 & \textbf{87.4} & 87.2 & \textbf{87.8}  \\
GLUE/QNLI (acc.\%) & 92.0 & \textbf{92.5} & \textbf{93.0} & 92.8 & 92.9 & \textbf{93.6} & \textbf{94.1} & \textbf{94.1}  \\
GLUE/QQP (F-1\%) & \textbf{71.5} & \textbf{71.5} & 71.8 & \textbf{72.0} & 71.8 & \textbf{71.9} & \textbf{72.6} & 72.3  \\
GLUE/SST-2 (acc.\%) & \textbf{93.5} & 93.2 & 94.1 & \textbf{94.4} & 94.0 & \textbf{95.6} & 95.6 & \textbf{95.9}  \\
\end{tabular}
\end{table*}

\begin{table*}[h]
\centering
\caption{Experiment results using BERT-base (12-layer with embedding size 768), on the English Wikipedia + BookCorpus dataset. The proposed SWE method also improves the performance of BERT-base.}
 \vspace{3mm}
\label{tab:swe_bert_base}
\begin{tabular}{l|cc|cc}
         & \multicolumn{4}{c}{BERT-base} \\ \cline{2-5}
         & \multicolumn{2}{c|}{0.5 m iterations} & \multicolumn{2}{c}{1 m iterations} \\
         & Baseline & SWE & Baseline & SWE \\ \hline
Pretrain MLM (acc.\%) & 68.74 & \textbf{69.21} & 69.86 & \textbf{70.16} \\ \hline
SQuAD average         & 82.06 & \textbf{82.91} & 85.01 & \textbf{85.50} \\
GLUE average          & 76.17 & \textbf{76.72} & 76.40 & \textbf{77.55} \\ \hline
SQuAD v1.1 (F-1\%)    & 88.29 & \textbf{89.23} & 90.32 & \textbf{90.40} \\
SQuAD v2.0 (F-1\%)    & 75.82 & \textbf{76.59} & 79.69 & \textbf{80.59} \\
GLUE/AX (corr\%)      & 35.3  & \textbf{36.2}  & 34.7  & \textbf{36.4}  \\
GLUE/MNLI-m (acc.\%)  & 83.9  & \textbf{84.3}  & 84.4  & \textbf{85.2}  \\
GLUE/MNLI-mm (acc.\%) & 83.0  & \textbf{83.5}  & 83.5  & \textbf{84.9}  \\
GLUE/QNLI (acc.\%)    & 90.9  & \textbf{91.6}  & 91.5  & \textbf{92.9}  \\
GLUE/QQP (F-1\%)      & 70.9  & \textbf{71.2}  & 71.4  & \textbf{71.5}  \\
GLUE/SST-2 (acc.\%)   & 93.0  & \textbf{93.5}  & 92.9  & \textbf{94.4}  \\
\end{tabular}
\end{table*}

We use the AdamW optimizer \citep{loshchilov2017decoupled} with the weight decay rate being $0.01$, $\beta_1=0.9$, and $\beta_2=0.999$. For English Wikipedia and BookCorpus, we use Pre-LN \citep{he2016identity,wang2019learning} instead of the original BERT's Post-LN. Note that the correct implementation of Pre-LN contains a final layer-norm right before the final classification/masked language modeling layer. Unlike the claim made by \cite{xiong2020layer}, we notice that using Pre-LN with learning rate warmup leads to better baseline performance, as opposed to not using learning rate warmup. In our implementation, the learning rate starts from $0.0$, linearly increases to the peak value of $3\times 10^{-4}$ (the same learning rate used by \cite{xiong2020layer}) at the $10k$-th iteration, and then linearly decays to $0.0$ at the end of the training. For the XLNet dataset, we apply the same Pre-LN setup except the peak learning chosen to be $2 \times 10^{-4}$. The peak learning rate of $3\times 10^{-4}$ makes training unstable here and yields worse performance than $2 \times 10^{-4}$.

We adopt the training procedure in the TensorFlow official implementation in which a BERT model is trained for 1 million iteration steps (both on English Wikipedia plus BookCorpus and on the XLNet dataset) with a constant batch size of 512. We also report the results of training for a half-million steps, to show how our method performs with less computational resources. \textit{When applying the proposed SWE, we train the model with weights shared for \textbf{10\%} of the total number of steps, then train with weights untied for the rest of the steps.}

\begin{table*}[h]
\centering
\caption{We group several consecutive layers as a weight sharing unit instead of sharing weights only across original layers. $A\times B$ means grouping $A$ layers as a unit which is being shared with $B$ times. Models are trained on English Wikipedia and BookCorpus. Note that SWE 4x6 and 2x12 trained for a half million iterations achieve similar fine-tuning results compared to the non-SWE baseline trained for one million iterations.}
\vspace{3mm}
\label{tab:logical_physical_layers}
\begin{tabular}{l|c|ccccc|c}
& Baseline & \multicolumn{5}{|c|}{SWE} & Baseline \\
& 0.5 m iter. & \multicolumn{5}{|c|}{0.5 m iterations} & 1 m iter. \\
& 24x1 & 12x2 & 6x4 & 4x6 & 2x12 & 1x24 & 24x1 \\ \hline
Pretrain MLM (acc.\%) & 73.66 & 73.82 & 73.99 & 73.90 & 74.16 & 73.92 & \textbf{74.98} \\ \hline
SQuAD average         & 87.27 & 87.96 & 88.59 & \textbf{89.17} & 88.75 & 88.34 & 88.82 \\
GLUE average          & 78.17 & 78.87 & 79.12 & 78.82 & \textbf{79.47} & 78.98 & 79.30 \\ \hline
SQuAD v1.1 (F-1\%)    & 91.54 & 92.18 & \textbf{92.62} & 92.52 & 92.44 & 92.54 & 92.58 \\
SQuAD v2.0 (F-1\%)    & 82.99 & 83.74 & 84.56 & \textbf{85.82} & 85.06 & 84.14 & 85.06 \\
GLUE/AX (corr\%)      & 39.0  & 40.7  & 40.3  & 40.2  & \textbf{43.0}  & 40.1  & 42.3  \\
GLUE/MNLI-m (acc.\%)  & 85.9  & 86.9  & \textbf{87.9}  & 87.0  & 87.1  & 87.2  & 86.9  \\
GLUE/MNLI-mm (acc.\%) & 85.4  & 85.8  & \textbf{87.0}  & 86.4  & 86.4  & 86.7  & 86.1  \\
GLUE/QNLI (acc.\%)    & 92.3  & 93.1  & 93.2  & 92.9  & \textbf{93.8}  & 93.4  & 93.5  \\
GLUE/QQP (F-1\%)      & 72.1  & 72.0  & 72.0  & 71.6  & 71.8  & 71.7  & \textbf{72.2}  \\
GLUE/SST-2 (acc.\%)   & 94.3  & 94.7  & 94.3  & \textbf{94.8}  & 94.7  & \textbf{94.8} & \textbf{94.8}  \\
\end{tabular}
\end{table*}

After pre-training, we fine-tune the models for the Stanford Question Answering Dataset (SQuAD v1.1 and SQuAD v2.0) \citep{rajpurkar2016SQuAD} and the GLUE benchmark \citep{wang2018glue}. For all fine-tuning tasks, we follow the setting as in the literature: the model is fine-tuned for 3 epochs; the learning rate warms up linearly from 0.0 to peak in the first 10\% of the training iterations, then linearly decay to $0.0$. We select the best peak learning rate from $\{1\times10^{-5}, 1.5\times10^{-5}, 2\times10^{-5}, 3\times10^{-5}, 4\times10^{-5}, 5\times10^{-5}, 7.5\times10^{-5}, 10\times10^{-5}, 12\times10^{-5}\}$ based on performance on the validation set. For the SQuAD datasets, we fine-tune each model 5 times and report the average. For the GLUE benchmark, for each training method, we simply train one BERT model and submit the model's predictions over the test sets to the GLUE benchmark website to obtain test results. We observed that when using Pre-LN, the GLUE fine-tuning process is stable and only rarely leads to divergence.

\subsection{Experiment Results}
Both pre-training and fine-tuning results of our method vs. the baseline method are shown in Table \ref{tab:swe_bert}. We see that our method consistently outperforms baseline methods, especially on fine-tuning tasks. We show training loss curves in Figure \ref{fig:loss_curve}. The training loss is high when the weights are shared. The loss drops significantly after untying. Eventually, SWE yields a slightly lower training loss than the non-SWE baseline.

\begin{figure}[h]
\centering
\includegraphics[clip, trim={25 15 5 25}, width=0.80\columnwidth]{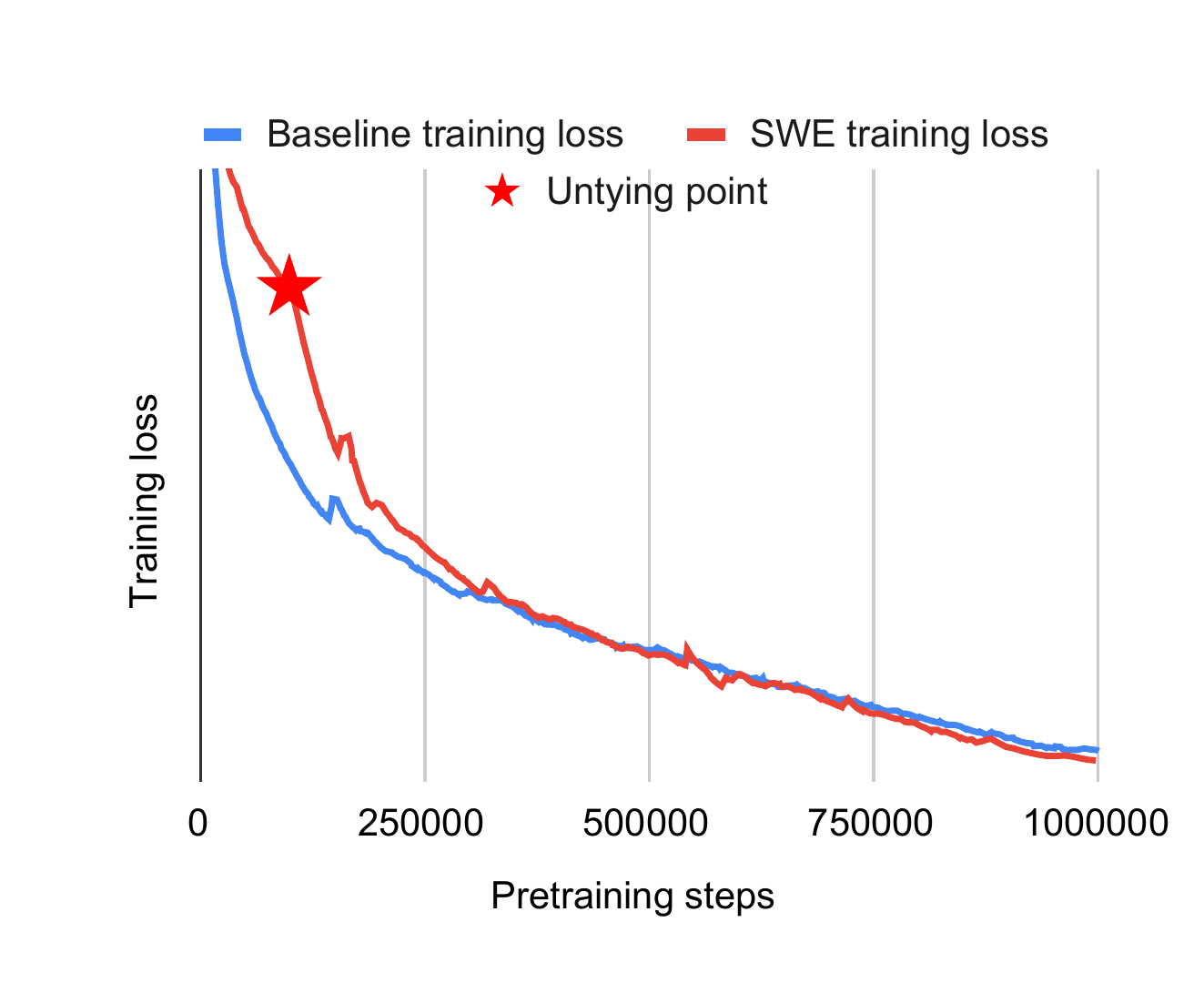}
\caption{BERT pre-training loss curves. In the first 100k steps of SWE training, weights are shared and the pre-training loss is high. After untying, the loss improves significantly. Eventually, SWE yields a slightly lower pre-training loss, and significantly better fine-tuneing results. This matches the observance in our illustrative example for linear regression (Figure \ref{fig:regression}).}
\label{fig:loss_curve}
\vspace{-5mm}
\end{figure}

\subsection{Ablation studies}
\label{sec:ablation_studies}
In this section, we study how our method performs across different experimental settings.

\subsubsection{How SWE Works with Different Model Sizes}
\label{sec:different_widths}
According to experiments shown in the ALBERT paper \cite{lan2020albert}, ALBERT-base with an embedding size of 128 performs preferably compared to ALBERT-base with a smaller or larger embedding size. In this section, we conduct experiments to see if the performance of our SWE training method is also related to the width of the model.

In particular, in addition to the conventional embedding size of 1024 in BERT, we experiment with alternative embedding sizes of 768 and 1536. We also scale the number of hidden units in the feed-forward block, and the size of each self-attention head accordingly. For each model size, we select the best peak learning rate from $\{1 \times 10^{-4}, 2 \times 10^{-4}, 3 \times 10^{-4}, 4 \times 10^{-4}\}$. The resulting peak learning rates adopted for embedding sizes of 768 and 1536 are $3 \times 10^{-4}$ and $2 \times 10^{-4}$, respectively. Experiment results are shown in Table \ref{tab:swe_embed_size}. The proposed SWE approach improves the performance of BERT with different embedding sizes.

Additionally, we experiment with the BERT-base architecture which contains 12 transformer layers, with an embedding size of 768. We keep other experimental settings such as the learning rate schedule unchanged. Results are shown in Table \ref{tab:swe_bert_base}. The proposed SWE method also improves the performance of the BERT-base consistently.

\subsubsection{When to Stop Weight Sharing}
\label{sec:when_to_unshare}
In this section, we study the effects of using different untying points (Algorithm \ref{alg:weight_sharing_fix}). If weights are shared throughout the entire pre-training process, the final performance will be much worse than without any form of weight sharing \citep{lan2020albert}. On the other hand, without weight sharing at all yields worse generalization ability.

Results of using different untying point $\tau$ values are summarized in Figure. \ref{fig:break_points}. BERT models are trained for 500k iterations on English Wikipedia and BookCorpus. From the results, we see that untying at around 10\% of the total number of training steps lead to the best performance. This is the reason we use $\tau=50k$ / $\tau=100k$ when training BERT for a half-million / one million steps, respectively.

\begin{figure}[h]
\centering
\includegraphics[clip, trim={0 20 0 20}, width=0.80\columnwidth]{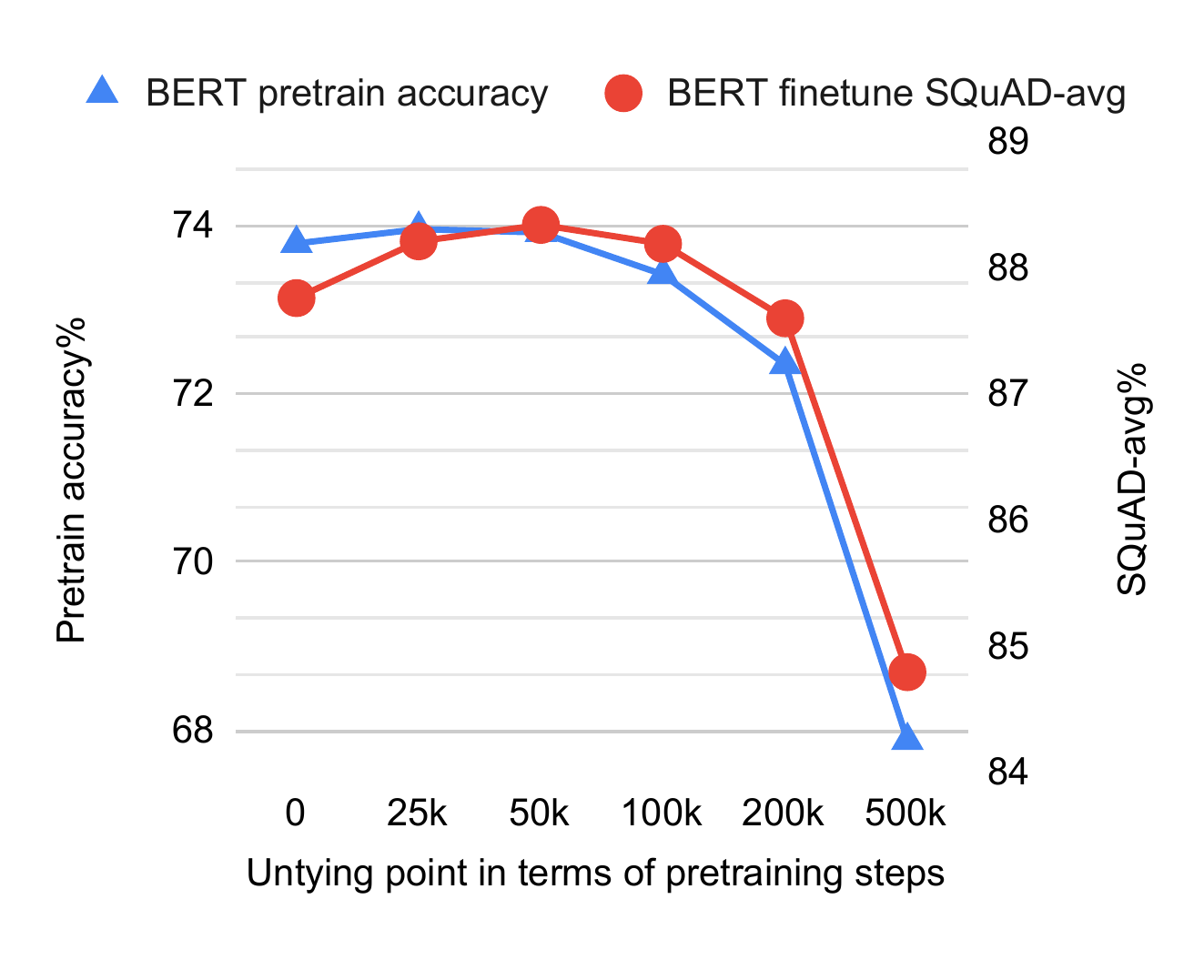}
\caption{Selecting the untying point $\tau$ from range $\tau=0$ (non-weight sharing baseline) to $\tau=T$ (ALBERT-like). The optimal untying point is around 10\% of the total number of training steps.}
\label{fig:break_points}
\end{figure}

\subsubsection{How to choose weight sharing units}
Note that it is not necessary to be restricted to share weights only across the original layers. We can group several consecutive layers as a weight-sharing unit. We denote $A \times B$ as grouping $A$ layers as a weight sharing unit which is being shared with $B$ times. Since BERT has 24 layers, the baseline method without weight sharing can be viewed as ``24x1'', and our method shown in Table \ref{tab:swe_bert} can be viewed as ``1x24''. We present results from more different choices of  weight sharing units in Table \ref{tab:logical_physical_layers}. We can see that,  in order to achieve good results, the size of the chosen weight-sharing unit should not be larger than 6 layers. This means that the weights of a layer must be shared for at least 4 times.

\section{Conclusion}
\label{sec:conclusion}

We proposed a simple weight sharing method to speed up the learning of the BERT model and showed promising empirical results. In particular, our method demonstrated consistent improvement for BERT training and its performance on various downstream tasks. Our method is motivated by the successes of weight sharing models in the literature, and is validated with theoretic analysis. For future work, we will extend our empirical studies to other deep learning models and tasks, and analyze under which conditions our method will be helpful. 

\bibliography{references}
\bibliographystyle{icml2021}

\newpage
\appendix
\section{Proofs of Section \ref{sec:theory}}

\subsection{Proof of Proposition \ref{prop:regression}}
\begin{proof}
    Let $\mathbbm{1}$ be the $L$-dimensional vector with all elements being 1. Let $\overline w^* = \Pcal_{\Wcal_0}(w^*), \widetilde w^* = \Pcal_{\Wcal_0^\perp}(w^*)$. For simplicity, denote $\overline w^* = w_0^* \cdot \mathbbm{1}$, and $\widehat w = w_0 \cdot \mathbbm{1}$, where $w_0^*, w_0$ are scalars. For the training set $\cbr{x_i, x_i^\top w^*}$, let the gradient of $x_0$ be 0, we have
    \begin{align*}
         & \frac{\partial}{\partial w_0}\sum_{i=1}^n\sbr{x_i^\top \mathbbm{1}\cdot w_0 - x_i^\top\rbr{w_0^*\cdot \mathbbm{1} + \widetilde w^*}}^2 = 0.
    \end{align*}
    It implies the solution of $w_0$ to be
    \begin{align*}
        w_0 = & \frac{ w_0^* \sum_{i=1}^n\rbr{x_i^\top \mathbbm{1}}^2 + \sum_{i=1}^n\rbr{x_i^\top \mathbbm{1}\cdot x_i^\top \widetilde \theta^*}}{\sum_{i=1}^n\rbr{x_i^\top \mathbbm{1}}^2} \\ 
        = & w_0^* + \frac{\sum_{i=1}^n\rbr{x_i^\top \mathbbm{1}\cdot x_i^\top \widetilde \theta^*}}{\sum_{i=1}^n\rbr{x_i^\top \mathbbm{1}}^2}
    \end{align*}
    Note that by definition, $\mathbbm{1}$ is orthogonal with $\widetilde \theta^*$. For $x_i$ generated from gaussian distribution with identity covariance, we know $x_i^\top \mathbbm{1}$ and $x_i^\top \widetilde \theta^*$ are independent gaussian random variable with variance $L$. Therefore we have
    \begin{align*}
        \abs{w_0 - w_0^*} = O\rbr{\frac{1}{\sqrt{n}}},
    \end{align*}
    which implies for the weight sharing solution $\widehat w$, we have
    \begin{align*}
        \norm{\widehat w - \overline w_0^*} = O\rbr{\sqrt{L / n}}.
    \end{align*}
\end{proof}

\subsection{Proof of Theorem \ref{thm:weight_sharing}}

\begin{proof}
    With weight sharing, we have
    \begin{align*}
        \dot W_l(t) = -\sum_{l=1}^L\nabla_l\Rcal(t) / L = -W_0^{L-1}(t)\rbr{W_0^L(t) - \Phi}.
    \end{align*}
    For the loss function $\Rcal(t)$, we have
    \begin{align*}
        \dot\Rcal(t) = & \sum_{l=1}^L\tr\rbr{\nabla_l^\top\Rcal(t)\dot W_l(t)} =  -\sum_{l=1}^L\norm{\nabla_l\Rcal(t)}_F^2 \\
        = & - \sum_{l=1}^L\norm{W_0(t)^{L-1}\rbr{W_0(t)^L-\Phi}}_F^2 \\
        \le & - 2L \lambda^2_\text{min}(W_0(t)^{L-1})\Rcal(t).
    \end{align*}
    By continuous gradient descent with $W_0(0) = I$, it is easy to see that
    \begin{align*}
        \lambda_\text{min}(W_0^{L-1}(t)) \ge \min(1, \lambda_\text{min}(\Phi)).
    \end{align*}
    Therefore we have
    \begin{align*}
        &\dot\Rcal(t) \le -2L\min(1, \lambda^2_\text{min}(\Phi))\Rcal(t) \\
        &\implies \Rcal(t) \le e^{-2L\min(1, \lambda^2_\text{min}(\Phi))t}\Rcal(0).
    \end{align*}
\end{proof}

\subsection{Technical Lemma}
\begin{lemma}\label{lemma:condition_on_eta}
  Initializing $W_0 = I$ and training with weight sharing update (Equation \ref{eqn:weight_sharing_update}), by setting $\eta \le \frac{1}{L \max\rbr{\lambda_\text{max}^2(\Phi), 1}}$, we have
  \begin{align*}
      & \lambda_\text{min}\rbr{W_0(t)} \ge \min(\lambda_\text{min}(\Phi)^{1/L}, 1),\\
      & \lambda_\text{max}\rbr{W_0(t)} \le \max(\lambda_\text{max}(\Phi)^{1/L}, 1).
  \end{align*}
\end{lemma}
\begin{proof}
    With weigth sharing, we know that $W_0(t)$ has the same eigenvectors as $\Phi$. Take any eigenvector, denote $\widehat\lambda(t)$ and $\lambda^L$ to be the corresponding eigenvalue of $W_0(t)$ and $\Phi$. Thus
    \begin{align*}
        \widehat\lambda(t+1) = \widehat\lambda(t) - \eta\widehat\lambda(t)^{L-1}\rbr{\widehat\lambda(t)^L - \lambda^L}.
    \end{align*}
    For $\lambda > 1$, we would like to show $\widehat\lambda(t) \in [1, \lambda], \forall t\ge 0$ by setting $\eta \le \frac{1}{L\lambda^{2L}}$. Since $\widehat\lambda(0) = 1$, we know this claim holds trivially at $t=0$. Then suppose we have the claim holds for $t = t_0$, then $\widehat\lambda(t_0+1)$ equals to
    \begin{align*}
         \widehat\lambda(t_0) - \eta\widehat\lambda(t_0)^{L-1}\rbr{\widehat\lambda(t_0)-\lambda}\sum_{i=0}^{L-1}\widehat\lambda(t_0)^i\lambda^{L-1-i}.
    \end{align*}
    To make $\widehat\lambda(t_0 + 1) \in [1, \lambda]$, we set $\eta \le \frac{1}{L \lambda ^{2L}}$ and can upper bound $\widehat\lambda(t_0+1)$ by
    \begin{align*}
    \widehat\lambda(t_0) - \frac{\widehat\lambda(t_0)^{L-1}\sum_{i=0}^{L-1}\widehat\lambda(t_0)^i\lambda^{L-1-i}}{\lambda^{L-1}\sum_{i=1}^{L-1}\lambda^{L-1}}\rbr{\widehat\lambda(t_0)-\lambda} \le \lambda.
    \end{align*}
    And $\eta \ge 0$ guarantees that $\widehat\lambda(t_0+1) \ge \widehat\lambda(t_0) \ge 1$. By induction, when $\lambda > 1$, we have$\widehat\lambda(t) \in [1, \lambda], \forall t \ge 0$.
    
    Similarly, for $\lambda < 1$, we would like to show $\widehat\lambda(t) \in [\lambda, 1]$ by setting $\eta \le \frac{1}{L}$. Note again the claim holds trivially when $t = 0$. Suppose $\widehat\lambda(t_0) \in [\lambda, 1]$, we can lower bound $\widehat\lambda(t_0+1)$ by
    \begin{align*}
         \widehat\lambda(t_0) - \frac{\widehat\lambda(t_0)^{L-1}\sum_{i=0}^{L-1}\widehat\lambda(t_0)^i\lambda^{L-1-i}}{1 \sum_{i=0}^{L-1}1} \rbr{\widehat\lambda(t_0)-\lambda} \ge \lambda.
    \end{align*}
    And $\eta \ge 0$ guarantees that $\widehat\lambda(t_0+1) \le \widehat\lambda(t_0) \ge 1$. By induction, when $\lambda < 1$, we have$\widehat\lambda(t) \in [\lambda, 1], \forall t \ge 0$.
    
    Note that the two claims hold for all $\lambda$, it then directly implies that by setting $\eta \le \frac{1}{L \max\rbr{\lambda_\text{max}^2(\Phi), 1}}$, we have
    \begin{align*}
        & \lambda_\text{min}\rbr{W_0(t)} \ge \min(\lambda_\text{min}(\Phi)^{1/L}, 1),\\
        & \lambda_\text{max}\rbr{W_0(t)} \le \max(\lambda_\text{max}(\Phi)^{1/L}, 1),
    \end{align*}
    which completes the proof.
\end{proof}

\subsection{Proof of Theorem \ref{thm:global_convergence}}
\begin{proof}
  Denote $W_{i:j} \coloneqq W_{i}W_{i-1}\cdots W_{j+1}W_{j}$. Focusing on the first stage of weight-sharing training, we have $W_{i} = \frac{1}{\sqrt{L}}\cdot \frac{1}{2}\rbr{\Gamma + \Gamma^{\top}}$. The gradient of $W_i$ for $i > 0$ is
  \begin{align*}
    \nabla_i \Rcal = \frac{\partial \Rcal(\widehat W_{i})}{\partial \widehat W_{i}} & = W_{L:i+1}^{\top}\rbr{W_{L:1} - \Phi}W_{i-1:1}^{\top} \\
    & = W_0^{L-i}\rbr{W_0^{L} - \Phi}W_0^{i-1}
  .\end{align*}
  which implies that the gradient of $W_0$ is
  \begin{align*}
    \nabla_{0} \Rcal & = \frac{1}{2n}\sum_{i=1}^{L}\nabla_i \Rcal(t) + \nabla_i \Rcal(t)^{\top} \\
    & = \frac{1}{L}\sum_{i=1}^{L}W_0^{L-i}\rbr{W_0^{L} - \frac{1}{2}\rbr{\Phi + \Phi^{\top}}}W_0^{i-1}
  \end{align*}
  Therefore, the first stage of weight-sharing is equivalent to training towards $\frac{1}{2}\rbr{\Phi + \Phi^{\top}}$. For simplicity, define $\Phi_{s} = \frac{1}{2}\rbr{\Phi + \Phi^{\top}}$, where $s$ stands for "symmetry". Further, $\Phi_{s}$ is positive definite as $\sigma_{min}(\Phi_{s}) \ge \sigma_{min}(\Phi) \ge \sigma_{min}(\Phi_{0}) - \frac{1}{3}\sigma_{min}(\Phi_{0}) > 0$. We have proved in Theorem \ref{thm:weight_sharing} and \ref{thm:discrete_weight_sharing} that learning a positive definite matrix with identity initialization
  by weight-sharing is easy and fast.

  For the $\widehat w$ at the end of the first stage in weight-sharing, we have $\norm{\widehat W_{i+1}^{\top}\widehat W_{i+1} - \widehat W_{i}\widehat W_{i}^{\top}}_{F} = 0, \forall i$, and $\sqrt{2 \Rcal(\widehat w)} \approx \norm{\Phi - \Phi_{s}}_{F} \le \norm{\Phi - \Phi_{0}}_{F} \le \frac{1}{3}\sigma_{min}(\Phi_{0})$. Combining with $\sigma_{min}(\Phi) \ge \frac{2}{3}\sigma_{min}\rbr{\Phi_{0}}$, we have $\sqrt{2 \Rcal(\widehat w)} \le
  \sigma_{min}(\Phi) - \frac{1}{3}\sigma_{min}(\Phi_{0})$. Thus $\widehat W$ satisfies all requirements for initialization in Fact \ref{fact:local_convergence}. And the convergence to $\Phi$ in the second stage of weight-sharing follows immediately.
\end{proof}

\subsection{Discrete-time gradient descent}\label{subsec:discrete_time}

One can extend the previous result to the discrete-time gradient descent with a positive constant step size $\eta$. It can be shown that with zero-asymmetric initialization, training with the gradient descent will achieve $\Rcal(t) \le \epsilon$ within $O(L^3\log(1/\epsilon))$ steps; initializing and training with weights sharing, the deep linear network will learn a positive definite matrix $\Phi$ to $\Rcal(t) \le \epsilon$ within $O(L\log(1/\epsilon))$ steps, which reduces the required iterations by a factor of $L^2$. Formally, see the convergence result of ZAS in Fact \ref{thm:discrete_no_sharing} and the convergence result of weight sharing in Theorem \ref{thm:discrete_weight_sharing}.

\begin{fact}[Continuous-time gradient descent without weight sharing \citep{wu2019global}]\label{thm:discrete_no_sharing}
  For deep linear network $f(\xb; W_1, ..., W_L) = W_LW_{L-1}...W_1\xb$ with zero-asymmetric initialization and discrete-time gradient descent, if the learning rate satisfies $\eta \le \min\cbr{\rbr{4L^3\xi^6}^{-1}, \rbr{144 L^2\xi^4}^{-1}}$, where $\xi = \max\cbr{2\norm{\Phi}_F, 3L^{-1/2}, 1}$, then we have linear convergence $\Rcal(t) \le \rbr{1 - \frac{\eta}{2}}^t\Rcal(0)$.
\end{fact}
Since the learning rate is $\eta = O(L^{-3})$, Fact \ref{thm:discrete_no_sharing} indicates that the gradient descent can achieve $\Rcal(t) \le \epsilon$ within $O(L^3\log(1/\epsilon))$ steps.

\begin{theorem}[Discrete-time gradient descent with weight sharing]\label{thm:discrete_weight_sharing} For the deep linear network $f(\xb; W_1, ..., W_L) = W_LW_{L-1}...W_1\xb$, initialize all $W_l(0)$ with identity matrix $I$ and update according to Equation \ref{eqn:weight_sharing_update}. With a positive definite target matrix $\Phi$, and setting $\eta \le \frac{\min(\lambda^2_\text{min}(\Phi), 1)}{4\sqrt{d}L^2\max\rbr{\lambda^4_\text{max}(\Phi), 1}}$, we have linear convergence $\Rcal(t) \le \exp\sbr{- (2L-2)\min\rbr{\lambda_\text{min}^2(\Phi), 1}\eta t}\Rcal(0)$.
\end{theorem}
Take $\lambda_\text{min}(\Phi)/\lambda_\text{max}(\Phi), d$ as constants and focus on the scaling with $L, \epsilon$, we have $\eta = O(L^{-2})$. Because of the extra $L$ in the exponent, we know that when learning a positive definite matrix $\Phi$, training with weight sharing can achieve $\Rcal(t) \le \epsilon$ within $O\rbr{L\log \rbr{1/\epsilon}}$ steps. The dependency on $L$ reduces from previous $L^3$ to linear, which shows the acceleration of training by weight sharing.

\subsection{Proof of Theorem \ref{thm:discrete_weight_sharing}}
\begin{proof}
    Denote $\phi = \max(\lambda_\text{max}(\Phi), 1)$. Training with weight sharing, we have
    \begin{align*}
        \nabla_l\Rcal(t) = W_0^{L-1}(t)\rbr{W_0^L(t)-\Phi},\quad 
    \end{align*}
    By Lemma \ref{lemma:condition_on_eta}, setting $\eta \le \frac{1}{L \max\rbr{\lambda_\text{max}^2(\Phi), 1}}$, we have
    \begin{align*}
        & \lambda_\text{min}\rbr{W_0(t)} \ge \min(\lambda_\text{min}(\Phi)^{1/L}, 1),\\
        & \lambda_\text{max}\rbr{W_0(t)} \le \max(\lambda_\text{max}(\Phi)^{1/L}, 1).
    \end{align*}
    Denote $ {\phi} = \max(\lambda_\text{max}(\Phi), 1)$, we immediately have
    \begin{align*}
        & \norm{\nabla_l\Rcal(t)}_F \le  {\phi}^{(L-1)/L} \sqrt{2\Rcal(t)},\\
        & \norm{\nabla_l\Rcal(t)}_F \ge\min(\lambda_\text{min}(\Phi), 1)^{(L-1)/L} \sqrt{2\Rcal(t)}.
    \end{align*}
    With one step of gradient update, we have
    \begin{align*}
        & \Rcal(t+1) - \Rcal(t) \\
        = & \frac{1}{2}\sbr{\norm{\rbr{W_0-\eta\nabla_0 \Rcal(t)}^L-\Phi}_F^2 - \norm{W_0 - \Phi}_F^2} \\
        = & \rbr{\rbr{W_0 - \eta\nabla_0\Rcal(t)}^L - W_0^L}\odot\rbr{W_0^L - \Phi} \\
          & + \frac{1}{2}\norm{\rbr{W_0 - \eta\nabla_0\Rcal(t)}^L - W_0^L}_F^2,
    \end{align*}
    where the $\odot$ denotes the element-wise multiplication. Let
    \begin{align*}
        & I_1 = \eta A_1\odot\rbr{W_0^L-\Phi},\quad I_2 = \sum_{k=2}^L\eta^kA_k\odot\rbr{W_0^L-\Phi},\\
        & I_3 = \frac{1}{2}\norm{\sum_{k=1}^L\eta^kA_k}_F^2,
    \end{align*}
    where the matrix $A_k$ comes from
    \begin{align*}
        \rbr{W_0 - \eta\nabla_0\Rcal(t)}^L = A_0 + \eta A_1 + \cdots + \eta^L A_L.
    \end{align*}
    We have
    \begin{align*}
        \Rcal(t+1) - \Rcal(t) \le I_1 + I_2 + I_3.
    \end{align*}
    Note that 
    \begin{align*}
        \norm{W_0}_2 \le {\phi}^{1/L}, \quad \norm{\nabla_0\Rcal(t)}_F \le  {\phi}^{(L-1)/L}\sqrt{2\Rcal(t)}.
    \end{align*}
    Using the fact that for $0 \le y \le x/L^2$, 
    \begin{align*}
        & (x+y)^L \le x^L + 2Lx^{L-1}y,\\
        & (x+y)^L \le x^L + Lx^{L-1}y + L^2x^{L-2}y^2.
    \end{align*}
    
    Take $\eta  \le \frac{1}{L^2  {\phi}\sqrt{2\Rcal(t)}}$, we have
    \begin{align*}
        \norm{\sum_{k=1}^L\eta^kA_k}_F & \le \rbr{ {\phi}^{1/L} + \eta {\phi}^{\frac{L-1}{L}}\sqrt{2\Rcal(t)}}^L -  {\phi} \\
        & \le 2 \eta L  {\phi}^2 \sqrt{2\Rcal(t)},
    \end{align*}
    and,
    \begin{align*}
        \norm{\sum_{k=2}^L\eta^kA_k}_F \le & \rbr{ {\phi}^{1/L} + \eta {\phi}^{\frac{L-1}{L}}\sqrt{2\Rcal(t)}}^L -  {\phi}  \\
         & - \eta L  {\phi}^{2(L-1)/L} \sqrt{2\Rcal(t)} \\
         \le & 2 \eta^2L^2 {\phi}^3 \Rcal(t).
    \end{align*}
    Thus we have
    \begin{align*}
        I_2 &\le \norm{\sum_{k=2}^L\eta^kA_k}_F\norm{W_0 - \Phi}_F \le 2\sqrt{2} \eta^2L^2\phi^2\Rcal(t)^{3/2} \\
        I_3 &\le \frac{1}{2}\norm{\sum_{k=1}^L\eta^kA_k}_F^2 \le 4 \eta^2L^2\phi^4\Rcal(t).
    \end{align*}
    For $I_1$, we directly have
    \begin{align*}
        I_1 = -\eta L \nabla_0\Rcal(t)\odot W_0^{L-1}\rbr{W_0^L-\Phi} = -\eta L \norm{\nabla_0\Rcal(t)}_F^2.
    \end{align*}
    Since we have $\norm{\nabla_0\Rcal(t)}_F \ge \min\rbr{\lambda_\text{min}\rbr{\Phi}, 1}\sqrt{2\Rcal(t)}$. Therefore by setting $\eta$ to be the minimum of $\frac{1}{L \max\rbr{\lambda_\text{max}^2(\Phi), 1}}$, $\frac{1}{\sqrt{2}L^2\phi\Rcal(t)^{1/2}}$, $\frac{\min(\lambda^2_\text{min}(\Phi), 1)}{2\sqrt{2}L^2\phi^2\Rcal(t)^{1/2}}$, and $\frac{\min(\lambda^2_\text{min}(\Phi), 1)}{4L^2\phi^4}$, we have
    \begin{align*}
        \Rcal(t+1) - \Rcal(t) \le (2-2L)\eta \min(\lambda^2_\text{min}(\Phi), 1)\Rcal(t).
    \end{align*}
    By directly setting
    \begin{align*}
        \lambda \le \frac{\min(\lambda^2_\text{min}(\Phi), 1)}{4\sqrt{d}L^2\max\rbr{\lambda^4_\text{max}(\Phi), 1}},
    \end{align*}
    we can satisfy all the requirements above, which will give us
    \begin{align*}
        \Rcal(t) & \le \sbr{1 - (2L-2)\min\rbr{\lambda_\text{min}^2(\Phi), 1}\eta}^t\Rcal(0) \\
        & \le \exp\sbr{- (2L-2)\min\rbr{\lambda_\text{min}^2(\Phi), 1}\eta t}\Rcal(t).
    \end{align*}
\end{proof}
\section{Experiments}

\subsection{Description of additional baseline methods}

\paragraph{Net2Net} This approach \cite{chen2015net2net} can be viewed as another form of sharing weights then unsharing -- before the Net2Net-based network growth, the narrow network is equivalent to a wide network with weights shared along the width dimension. We compare to Net2Net, to show whether sharing weights along the width dimension is better than sharing weights across layers. In particular, we train a BERT-large model with half of its original width for 50k iteration steps, then grow it to the full size BERT according to the Net2Net approach, and train it for 450k steps. This is similar to our SWE approach with which weights are shared for the first 50k steps, then unshared for 450k steps.

\paragraph{Reparameterization} Instead of weight sharing, we can reparameterize the weights of each layer \cite{lippl2020iterative} according to Equation \ref{eq:reparameter}. In this form, the final gradient of each layer is a linear combination of the averaged gradient, and the layer's original gradient. In particular, we feed the following gradient to the optimizer: $$\widetilde g_{i} = \dfrac{g_{0} + g_{i}}{2}\text{,}$$ where $g_{0}$ is the layer-wise averaged gradient and $g_{i}$ is the gradient of any layer.

\subsection{Experiment results}

We show results of training BERT-large for 0.5 million steps in Table \ref{tab:additional_exp}. Our proposed method consistently outperforms these baseline methods.
\begin{table}[h]
\centering
\caption{Comparing to Net2Net \cite{chen2015net2net} and weight reparameterization \cite{lippl2020iterative}. All models are trained for 0.5 million steps.}
 \vspace{5mm}
\label{tab:additional_exp}
\begin{tabular}{l|ccc}
                     & Net2Net & Repara. & SWE \\ \hline
Pretrain MLM (acc.\%)& 72.34 & 73.57 & \textbf{73.92} \\ \hline
SQuAD average        & 86.23 & 87.57 & \textbf{88.34} \\
GLUE average         & 77.55 & 78.35 & \textbf{78.98} \\ \hline
SQuAD v1.1 (F-1\%)   & 90.72 & 91.55 & \textbf{92.54} \\
SQuAD v2.0 (F-1\%)   & 81.73 & 83.59 & \textbf{84.14} \\
GLUE/AX (corr\%)     & 37.7  & 39.3  & \textbf{40.1}  \\
GLUE/MNLI-m (acc.\%) & 85.1  & 86.5  & \textbf{87.2}  \\
GLUE/MNLI-mm (acc.\%)& 84.5  & 85.7  & \textbf{86.7}  \\
GLUE/QNLI (acc.\%)   & 92.2  & 92.2  & \textbf{93.4}  \\
GLUE/QQP (F-1\%)     & 71.3  & \textbf{72.2}  & 71.7  \\
GLUE/SST-2 (acc.\%)  & 94.5  & 94.2  & \textbf{94.8}  \\
\end{tabular}
\end{table}

\end{document}